\documentclass[11pt]{article}

\usepackage[final]{acl}

\usepackage{times}
\usepackage{latexsym}
\usepackage{microtype}
\usepackage[T1]{fontenc}

\usepackage[utf8]{inputenc}

\usepackage{microtype}

\usepackage{inconsolata}

\usepackage{graphicx}
\usepackage{subcaption}
\usepackage[most]{tcolorbox}
\usepackage{booktabs}  

\title{Fathom-Vaidya: Advancing Medical Reasoning with Rubric-Based Rewards}

\author{
  Kalash Shah\textsuperscript{*} \\ \texttt{kalash.shah@fractal.ai}
  \And
  Kunal Singh\textsuperscript{*} \\ \texttt{kunal.singh@fractal.ai}
  \And
  Snehan J \\ \texttt{snehan.j@fractal.ai}
  \AND
  Shreyas Singh \\ \texttt{shreyas.singh@fractal.ai}
  \\[4pt]
  Fractal AI Research
}

\usepackage{amsmath, amsthm}

\newtheorem{proposition}{Proposition}
\usepackage{amssymb}
\usepackage{threeparttable}
\usepackage[svgnames, table, dvipsnames]{xcolor}
\usepackage{float}
\usepackage{arydshln}
\usepackage[hypcap=false]{caption}
\newcommand{\synhb}{Vaidya-Rubrics }

\begin{document}
\maketitle
{\renewcommand{\thefootnote}{}%
\footnotetext{\textsuperscript{*}\,Equal contribution}}
\begin{abstract}
Deploying Large Language Models (LLMs) in healthcare requires robust performance across two complementary dimensions - \textbf{diagnostic reasoning}: the convergent, evidence-driven task of inferring a patient's condition from clinical data to produce a diagnosis, and \textbf{clinical healthcare reasoning}: the broader, navigational judgment required to communicate, plan, and adapt across multi-turn clinical interactions where a single correct answer may not exist. Recent benchmarks such as HealthBench and MedXpertQA reveal persistent weaknesses in both areas, exposing failures in complex diagnostic scenarios and limitations in contextual, patient-centered dialogue. We introduce a sequential training framework that targets these facets using synthetic data and rubric-based reinforcement learning. First, we improve diagnostic reasoning using MedBullets-derived questions with rule- and rubric-guided Reinforcement Learning (RL). We then shift to clinical reasoning by generating 5.3k synthetic multi-turn scenarios, each paired with multi-dimensional rubrics to comprehensively assess the response. This approach yields over 10\% improvement on MedXpertQA, and our 30B model achieves 50.1\% accuracy on HealthBench-Hard, surpassing proprietary baselines including GPT-5 (thinking). Our results show that targeted synthetic datasets and rubric-based training can systematically improve both diagnostic and interactive clinical reasoning in medical LLMs.

\end{abstract}
\section{Introduction}
The growing deployment of LLMs in healthcare places stringent demands on their ability to reason reliably in real-world clinical settings, beyond performance on exam-style medical benchmarks. While prior work \cite{2025II-Medical-8B},~\cite{chen2024huatuogptiionestagetrainingmedical},~\cite{alhamed}) has largely focused on reinforcement learning (RL) with verifiable rewards for well-posed, single-answer examination questions, authentic clinical interactions involve multiple, interdependent dimensions such as accuracy, completeness, contextual awareness, and communication quality, rendering binary evaluation signals insufficient \cite{shool2025systematic}. Recent challenging benchmarks, including HealthBench-Hard \cite{arora2025healthbenchevaluatinglargelanguage} and Healthbench-Professional \cite{hicks2026healthbench}, highlight substantial gaps in current LLMs’ clinical reasoning abilities, underscoring the need for training and evaluation paradigms that better reflect the complexity of clinical decision-making.

Recent works have explored rubric-based RL using subjective evaluation criteria. Baichuan-M2 \cite{m2team2025baichuanm2scalingmedicalcapability} constructs synthetic clinical conversations from de-identified, publicly available data and employs a clinical evaluator to generate corresponding rubrics, which are then used to train a Qwen3-32B model via RL. Their study primarily targets diagnostic scenarios, with limited coverage of themes such as emergency referrals, global health, etc and the resulting rubrics emphasize accuracy and communication quality while largely omitting dimensions such as context-seeking, completeness, and instruction following. II-Medical-8B \cite{2025II-Medical-8B} applies RL to hard medical questions but does not incorporate multi-dimensional rubrics for evaluation. Currently, no existing dataset reflects the distribution of real-world, day-to-day clinical conversations between users and language models, apart from Healthbench. Very few studies have systematically used multi-dimensional rubrics to guide reinforcement learning for medical reasoning tasks, as most approaches train models to optimize only for the correctness of the final answer rather than the quality of the underlying reasoning process.

To address this challenge, we propose a sequential-training framework for improving diagnostic reasoning and clinical healthcare reasoning using synthetic data curation and multi-dimensional rubric generation that provides dense and informative training signals for RL in medical settings. 

We introduce a novel training approach for diagnostic reasoning. In this setting, we first employ a simple rule-based reinforcement learning scheme tailored to structured, single-answer medical questions. We curate a training corpus by extracting high-quality questions from Medbullets \cite{medbullets}, an online medical learning platform. Subsequently, we perform a second stage of online RL using error-rubrics to further refine the reasoning quality. This two-stage training process results in improved performance on across various diagnostic reasoning benchmarks, suggesting that RL can enhance clinical medical reasoning without relying on supervised Chain-of-Thought data.

Next, we extend our diagnostically-enhanced model to clinical healthcare reasoning. Unlike contemporary works~\cite{ye2025selfrewardingrubricbasedreinforcementlearning} that train directly on HealthBench-Easy, we also leverage MedRedQA~\cite{nguyen-etal-2023-medredqa}, a dataset grounded in real-world patient queries. We generate multi-turn synthetic conversations that simulate diverse user personas, which are iteratively refined via an actor–critic feedback loop to ensure completeness and consistency. Using a similar actor–critic process, we generate multi-dimensional evaluation rubrics aligned with these conversations, resulting in a synthetic dataset of 5,351 datapoints. We also train Qwen3-4B, Qwen3-8B, and Qwen3-30B-A3B models on this dataset, and refer to the resulting fine-tuned models as the Fathom-Vaidya (FV) family. Across all model scales, FV achieves consistent improvements exceeding 20\% on HealthBench-Hard~\cite{arora2025healthbenchevaluatinglargelanguage} and 12-15\% on Healthbench-Professional ~\cite{hicks2026healthbench} (a different benchmark to assess how helpful are LLMs in assisting doctors in their day-to-day tasks). FV-30B-A3B achieves a global best of \textbf{50.1\%} on Healthbench-hard. On Healthbench-professional, FV-30B-A3B approaches the physician baseline and surpasses Grok 4.2.


Our main contribution are as follows:
\begin{enumerate}
    \item We present a novel actor-critic framework for synthesizing realistic clinical healthcare conversations at scale. In a randomized clinical verification study, 99.4\% of the generated dialogues were judged to be medically valid, demonstrating the reliability of our generation pipeline.

    \item We propose an ICL driven actor-critic method for automatically producing multi-dimensional evaluation rubrics for clinical conversations, leading to the creation of the \synhb dataset. Training with \synhb enables rubric-guided RL that improves medical reasoning quality across model sizes.

    \item We introduce a training pipeline for challenging diagnostic reasoning, leveraging high-quality questions extracted from MedBullets. The approach combines an initial stage of rule-based RL with a subsequent rubric-based RL, and yields an improvement of approximately 11\% over Qwen3-8B without relying on supervised Chain-of-Thought data.

\end{enumerate}
\section{Preliminaries}
\subsection{Healthbench}
Healthbench~\cite{arora2025healthbenchevaluatinglargelanguage} released by OpenAI in May 2025, consists of 5,000 multi-turn conversations between a model
and majorly individual users or healthcare staff in some cases. The objective of this benchmark is to assess the final response generated by the model concluding the conversation. Each response is assessed (by an llm-as-a-judge) according to rubrics designed by medical experts, comprising subjective criteria and associated scores deemed appropriate by those experts. The entire dataset is split into two-subsets: healthbench-easy (4000 samples) and healthbench-hard (1000 samples), the latter consisting of the 1000 samples on which various off-the-shelf LLMs scored the least.
\subsection{MedXpertQA}
MedXpertQA~\cite{zuo2025medxpertqabenchmarkingexpertlevelmedical} is a recently released, challenging medical question answering benchmark. In this work, we focus exclusively on the MedXpertQA-Text subset, which contains hard, exam-level multiple-choice questions curated from authoritative medical examinations and specialty board questions across a diverse range of clinical domains. This subset contains 2,450 questions covering clinical subtasks like diagnosis, treatment planning, and basic medical knowledge. The questions also cover 11 body systems including cardiovascular, respiratory, nervous, digestive, and endocrine systems. The questions test clinical reasoning by requiring models to combine patient symptoms, clinical findings, and test results, rather than relying on isolated medical facts. 
\subsection{Group Sequence Policy Optimisation (GSPO)}
Group Sequence Policy Optimization (GSPO)~\cite{zheng2025groupsequencepolicyoptimization} optimizes policies using
sequence-level rewards by comparing multiple sampled trajectories within a
group. For a given context, we sample a group of $K$ sequences from a
reference policy and compute a normalized advantage for each sequence based
on its reward relative to the group. Policy updates are performed at the
sequence level using a clipped importance-weighted objective, which
encourages higher probability for sequences with positive relative advantage
while limiting large deviations from the reference policy. The GSPO loss is
defined in Equation~\ref{eq:gspo_loss}:


\begin{equation}
\label{eq:gspo_loss}
\begin{aligned}
\mathcal{L}(\theta)
&=
\mathbb{E}_{\mathcal{G}}
\Bigg[
\frac{1}{K}
\sum_{k=1}^{K}
\min \Bigg(
\\
&\quad
r_\theta(\tau^{(k)}) \hat{A}^{(k)},
\\
&\quad
\mathrm{clip}\!\left(
r_\theta(\tau^{(k)}),\, 1-\epsilon,\, 1+\epsilon
\right)
\hat{A}^{(k)}
\Bigg)
\Bigg].
\end{aligned}
\end{equation}

Here, $\mathcal{G}$ denotes a group of $K$ trajectories
sampled from the reference policy $\pi_{\theta_{\text{old}}}$ under the same
context. Each trajectory $\tau^{(k)}$ corresponds to a complete action
sequence. The term $\hat{A}^{(k)}$ is the group-normalized advantage of
$\tau^{(k)}$, computed from its scalar sequence-level reward by subtracting
the group mean and dividing by the group standard deviation. The importance
ratio $r_\theta(\tau^{(k)})$ measures the relative likelihood of the sequence
under the current policy $\pi_\theta$ compared to the reference policy, and
the clipping parameter $\epsilon$ controls the maximum allowed policy update.
Expectations are taken over groups sampled during training.
\section{\synhb Dataset}
\label{syn_dataset}
In order to construct a synthetic dataset containing multi-turn clinical conversations whose final responses are non-verifiable, we develop an actor-critic feedback loop to generate the theme-based conversations (Section~\ref{theme_convos}) and come up with their corresponding rubrics (Section~\ref{rubrics}). We refer to the actor LLM as $\mathcal{M}_{actor}$ and critic LLM as $\mathcal{M}_{critic}$. For our experiments, we use \textit{gpt-5-thinking} as $\mathcal{M}_{actor}$ and \textit{Grok 3} as $\mathcal{M}_{critic}$.
\subsection{Theme-based Conversations}
\label{theme_convos}
We adopt two different strategies to generate the raw conversations depending on the themes:
\begin{enumerate}
    \item For conversations belonging to the themes: responding-under-uncertainty and context-seeking, we source scenarios from the MedRedQA dataset \cite{nguyen-etal-2023-medredqa}. MedRedQA consists of medical queries posted by users on the \texttt{r/askdocs} subreddit, paired with responses that have undergone clinical verification. Conditioned on the target theme, $\mathcal{M}_{critic}$ analyzes the corresponding MedRedQA case and determines the number of dialogue turns as well as the intended structure of each turn. This specification is then provided to $\mathcal{M}_{actor}$, which generates the initial version of the conversation.
    
    \item For conversation themes such as global health, health-data interpretation, emergency referrals, and expertise-tailored communication, we instruct $\mathcal{M}_{actor}$ to generate raw conversations using In-Context Learning (ICL) examples drawn from HealthBench-Easy.

\end{enumerate}

We simulate both single-turn and multi-turn interactions between the model and the user. To generate diverse and realistic patient queries, we prompt $\mathcal{M}_{actor}$ to exhibit four distinct personality profiles, defined by the axes of introversion–extroversion and logical–emotional reasoning. The raw conversations are analysed by $\mathcal{M}_{critic}$ for (a) completeness: ensuring all the messages are complete and (b) consistency: ensuring the conversation is consistent in accordance with the theme. The actor-critic feedback loop is executed maximum 3 times until the conversation is deemed complete and consistent by $\mathcal{M}_{critic}$. All the prompts are described in Appendix~\ref{appendix:syn_conv_prompts}. 

We synthesize a total of 5351 datapoints wherein the distribution across different themes is shown in Figure~\ref{fig:synhb_theme_dist}.
\begin{figure*}[t]
    \centering
    \includegraphics[width=0.65\textwidth]{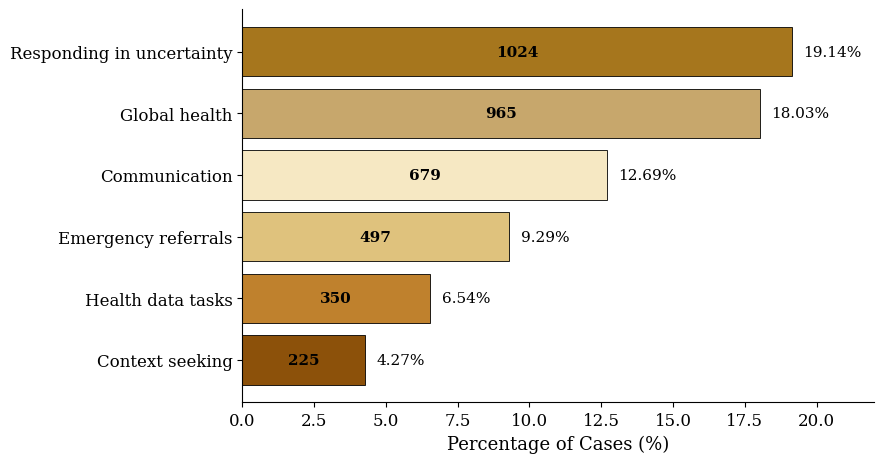}
    \caption{Distribution of cases across themes in \synhb}
    \label{fig:synhb_theme_dist}
\end{figure*}

\subsection{Rubrics}
\label{rubrics}
We construct an actor-critic feedback loop to generate evaluation rubrics for the conversations described in Section ~\ref{theme_convos}. Each rubric consists of a criterion and a continuous score in $[-10, +10]$, generated per theme--axis pair to ensure fine-grained, conversation-specific evaluation.

A key challenge in synthetic rubric generation is maintaining scoring consistency and clinical fidelity without human annotation at scale. Previous studies ~\cite{chen2021evaluatinglargelanguagemodels}, ~\cite{cobbe2021trainingverifierssolvemath} have shown that using ICL examples as anchors can significantly improve the performance of LLMs and their ability to generate coherent synthetic data. To address this, we ground the rubric generation process using a small set of examples from HealthBench-Easy, as \textit{calibration anchors} that expose $\mathcal{M}_{\text{actor}}$ to the scoring norms and clinical judgment standards validated by physicians. The examples are picked to ensure thematic overlap so that the calibration signal is domain-relevant.

Crucially, $\mathcal{M}_{\text{actor}}$ is explicitly instructed to generate rubrics that are \textbf{specific to the given conversation and theme--axis pair}, rather than to reproduce or paraphrase the retrieved examples. The retrieved rubrics serve solely to communicate the expected quality bar---analogous to how a few expert-annotated examples in any annotation pipeline help calibrate rater judgment without dictating rater output. A maximum of 5 rubrics are permitted per axis to prevent redundancy and encourage criterion diversity.
The generated rubrics are then evaluated by the critic model, $\mathcal{M}_{critic}$, which performs the following checks:
\begin{enumerate}
\item \textbf{Quality check}: Each rubric criterion must be complete, well-defined, and semantically distinct from the others.
\item \textbf{Score sanity check}: The score assigned to each rubric must be logically consistent with the rubric criterion.
\end{enumerate}

If either validation check fails, the critic provides structured feedback, and the resulting output is returned to $\mathcal{M}_{actor}$ for refinement. This actor–critic loop is executed for at most three iterations, or until both checks are satisfied, whichever occurs first. Rubrics that pass both checks are accepted and incorporated into the final synthetic dataset. For each conversation, rubrics across all evaluation axes are aggregated to form a complete datapoint. Rubrics for 288 datapoints (5.38\%) were modified by the critic in this process. The prompts are detailed in Appendix~\ref{appendix:rubrics_gen_prompts}.



    

    




Rubrics may exhibit either positive or negative polarity. Rubrics with negative polarity specify criteria corresponding to behaviors that the model should avoid. For instance, in the context of a heart attack, a positively polarized rubric may state, “Advises the patient to call an ambulance immediately,” whereas a negatively polarized rubric may specify, “Fails to advise the patient to call an ambulance immediately.” Rubrics with negative polarity are assigned negative scores. We posit that conversations containing a higher proportion of negatively polarized rubrics tend to receive lower overall scores; a detailed analysis is provided in Appendix~\ref{appendix:neg_rubrics}. To further increase the difficulty of the synthetic dataset, we randomly invert the polarity of a positive rubric with a probability of 10\%. Upon inversion, the rubric criterion is replaced by its negation, and its score is negated relative to the original value.

The synthetic dataset was subjected to a randomised clinical verification by a team of medical professionals and the evaluation is described in Appendix~\ref{appendix:clinical_eval}.

\section{Methodology}

\subsection{RL Training for Diagnostic Reasoning}

\subsubsection{Training Dataset}
\label{medbullets}


Diagnostic reasoning in medicine involves complex, multi-step decision making grounded in domain knowledge. Accordingly, we seek training data that reflects this level of difficulty and underlying reasoning structure. Previous work has trained medical reasoning models using established QA datasets, including exam-oriented benchmarks such as MedQA and MedMCQA~\cite{pmlr-v174-pal22a}, as well as biomedical QA datasets such as PubMedQA~\cite{jin-etal-2019-pubmedqa}. These datasets underpin systems such as Huatuo-o1~\cite{chen2024huatuogpto1medicalcomplexreasoning}, MedReason~\cite{wu2025medreasonelicitingfactualmedical}, and m1~\cite{huang2025m1unleashpotentialtesttime}.

Prior work by ~\cite{chen-etal-2025-benchmarking} curated USMLE Step 2/3--style questions from MedBullets, a widely used USMLE preparation platform, showing that such questions are inherently challenging and capture the depth of reasoning required for diagnosis. Building on this line of work, we incorporate a substantially larger corpus of MedBullets questions, which are openly available from the MedBullets website. Using a web crawler, we collect $7{,}801$ text-based questions spanning USMLE Step~1, USMLE Step~2, and Orthopedics question sets. Each question includes five answer options along with detailed explanations covering both correct and incorrect choices. The resulting dataset spans a broad range of domains, with approximately 35\% of questions covering foundational basic sciences, 30\% focusing on organ system--based medicine, 20\% on pathology and disease mechanisms, and the remainder distributed across reproductive medicine (7\%), psychiatry and behavioral science (5\%), and general topics (3\%). Its recency further makes it more challenging than earlier benchmarks such as MedQA~\cite{jin2020diseasedoespatienthave}. We use this MedBullets dataset alongside MedQA, MedMCQA to construct our overall training corpus. Overall, this dataset provides broad coverage and complexity aligned with the level of reasoning required for diagnosis.

\subsubsection{Rule-Based RL}
\label{train_strategy}

We first train the model using rule-based RL, without relying on supervised Chain-of-Thought data. Qwen3-8B is used as the base model and is optimized using the GSPO algorithm. Training is performed exclusively on multiple-choice questions, which provide a simple and reliable reward signal through their structured answer format. During training, a binary reward is assigned, with a reward of 1 for a correct answer and 0 otherwise. This is followed by a second stage of rubric-based training, as described in Section~\ref{rubrics_diag}. 

Prior to RL, we apply a difficulty-based filtering step using the base model. Each question is evaluated eight times at a temperature of 1.0, and questions that are either too easy or too difficult are discarded. This filtering step ensures that retained training examples provide informative learning signals. Let \(\bar{s}_i\) denote the average correctness score of question \(i\) across eight runs. Furthermore, we partition the retained data based on difficulty, datapoints with $\bar{s}_i \in [0.1, 0.4]$ are used for rubric-based RL, while those with $\bar{s}_i \in (0.4, 0.7]$ are used for rule-based RL. The final training dataset composition is shown in Table~\ref{tab:dataset_filtering}.

\begin{table}[t]
\centering
\small
\caption{Dataset composition and filtered data allocation for Qwen 3 8B}
\label{tab:dataset_filtering}

\begin{tabular}{lccc}
\hline
Dataset & \shortstack{Total\\Datapoints} & \shortstack{Rule-based\\RL} & \shortstack{Rubric\\RL} \\
\hline
MedQA & 10{,}178 & 1{,}127 & 766 \\
MedBullets & 7{,}801 & 1{,}192 & 570 \\
MedMCQA & 7{,}000 & 528 & 332 \\
\hline
\end{tabular}

\end{table}

\subsubsection{Online Rubric-Guided RL}
\label{rubrics_diag}

Building on the model trained with rule-based RL (Section~\ref{train_strategy}), we further optimize it using rubric signals derived from online RL. The resulting diagnostic model is denoted as \textbf{FV-8B-Diag}. To improve medical reasoning, we employ an online RL framework in which the \textit{GPT-5 (thinking mode)} model analyzes rollout trajectories to identify and categorize errors. Specifically, three classes of errors are considered: (i) knowledge errors (e.g., hallucinations), (ii) faulty causal reasoning, and (iii) context errors. Each identified error is assigned a severity level: \textit{major} errors receive a severity score of 2, while \textit{minor} errors receive a score of 1.

Based on the aggregated severity across a rollout, we define the rubric-based reward as:
\begin{equation}
    R_{\text{rubric}} = \frac{15 - \sum \text{severity}}{15}.
\end{equation}

We establish a normalization threshold of 15 based on empirical calibration: FV-8B, evaluated on 50 random samples from MedXpertQA, achieved a mean severity score of 11.4. This rubric reward is combined with an accuracy-based reward to form the final training signal:
\begin{equation}
    R_{\text{final}} = \alpha \cdot R_{\text{rubric}} + (1 - \alpha) \cdot R_{\text{accuracy}},
\end{equation}

where $\alpha \in [0,1]$ controls the trade-off between reasoning quality and answer correctness. Based on ablations (Section~\ref{ablation_medb}), we choose $\alpha=1$. Clinical verification of the error rubrics produced by \textit{GPT-5 (thinking)} has been presented in Appendix~\ref{appendix:clinical_eval}.
\subsubsection*{Training Recipe}
Qwen3-8B is trained using the verl framework in bfloat16 precision, with a batch size of 32, and max completion tokens of $10,000$. We used a cosine scheduler (5 warmup steps) with a peak learning rate of $1e^{-6}$ utilising the Adam optimiser ($\beta_{1}=0.9$, $\beta_{2}=0.95$).

 
\subsection{Rubrics based RL Training for Clinical Healthcare Reasoning}
\label{subsec:clinic}
To improve clinical reasoning, we train on top of \textbf{FV-8B-Diag}. In addition, to evaluate the scalability of our approach across model sizes, we train two separate models from the Qwen-3 series—Qwen3-4B and Qwen3-30B-A3B \cite{yang2025qwen3technicalreport}. This allows us to assess the effectiveness of our strategy across both smaller and larger model configurations. These models are trained using the GSPO algorithm. For each rubric criterion (Section~\ref{rubrics}) $j \in \{1, 2, \ldots, M_i\}$, a llm-as-a-judge (gpt-4.1) grades whether the rubric criterion is met or not, based on the conversation, the model response, and the criterion. We compute the final score by dividing the sum of points for criteria met by the maximum possible points in that example. For criterion $j$, take $\mathbf{1}_{\{r_{ij}\}}$ to be an indicator representing whether criterion $j$ is met and $p_{ij} \in [-10, 10], p_{ij} \neq 0$ to be its assigned point value. Then, the final score $s_{i}$ for that example is described in Equation ~\ref{eqn:hb_reward}:
\begin{equation}
    \label{eqn:hb_reward}
    s_i \;=\;
\frac{\displaystyle \sum_{j=1}^{M_i} \mathbf{1}_{\{r_{ij}\}}\, p_{ij}}
     {\displaystyle \sum_{j=1}^{M_i} \max(0,\, p_{ij})}
\end{equation}

We adopt the conversation score as the reward signal for training. Through empirical experiments, we have observed that adding a system prompt which guides the llm to answer based on the definition of rubric axes yields a 3-5\% absolute improvement (and often more) on healthbench-hard. Thus we adopt the system-prompt while training and the same is provided in Appendix~\ref{appendix:hb_prompts}. To ensure meaningful feedback for reinforcement learning, we filter the
training dataset at the beginning of each epoch and retain only datapoints whose average score over eight generations lies
within the interval $[0.15, 0.7]$.



\subsubsection*{Training Recipe}
\label{training_recipe}
We trained \textbf{Qwen3-4B} for 3 epochs, \textbf{FV-8B-Diag} for 5 epochs and \textbf{Qwen3-30B} for 3 epochs, depending on the saturation shown by the respective LLMs. All the models were trained using the verl framework. The LLMs were trained in $bfloat16$ data-format, with a train-batch size of $32$, context length of $20,000$ and max completion tokens of $16,000$. We used a cosine scheduler ($5$ warmup steps) with a peak learning rate of $1e^{-6}$ utilising the Adam optimiser ($\beta_{1}=0.9$, $\beta_{2}=0.95$). Following this training procedure, we refer to the resulting models collectively as \textbf{Fathom-Vaidya (FV)} $-$ namely FV-4B, FV-8B, and FV-30B-A3B.

\section{Results}

\subsection{Diagnostic Reasoning}

\begin{table}[ht]
\centering
\caption{Accuracy (\%) of LLMs on MedxpertQA.}
\begin{tabular}{l|c}
\hline
\textbf{Model} & \textbf{MedX Score} \\
\hline
Ultramedical-3.1-8B & 17.2 \\
Qwen3-8B & 17.4 \\
Huatuo-o1-8B & 18.9 \\
Qwen3-14B & 22.7 \\
Alphamed-8B & 22.9 \\
Qwen3-30B-A3B & 23.3 \\
II-Medical-8B & 25.0 \\ 
Medgemma-27B & 25.3 \\
Baichuan-M2-32B & 27.8 \\
\rowcolor{blue!8}
\textbf{FV-8B-Diag} & \textbf{27.9} \\
\rowcolor{blue!8}
\textbf{FV-8B} & \textbf{28.3} \\
Gpt-oss-20B & 28.9 \\
Qwen3-Next-80B-A3B & 32.5 \\
Gpt-oss-120B & 36.8 \\

\hline
\end{tabular}

\label{tab:medx_scores}
\end{table}






Table \ref{tab:medx_scores} reports the performance of LLMs on MedXpertQA. As shown, FV-8B outperforms other evaluated models in the 10B parameter range, achieving a score of \textbf{28.3}. Notably, unlike prior models that rely on SFT, our results are obtained using only RL, indicating that substantial gains can be achieved by effectively leveraging the base model’s existing knowledge. Compared to the Qwen3-8B base model, FV-8B achieves an absolute improvement of 10.9\%.

Table \ref{tab:medqa_mmlu_scores} presents the evaluation on additional diagnostic reasoning benchmarks: (a) MedQA, consisting of 1,273 questions from the test set (b) Medcase Reasoning \cite{wu2025medcasereasoningevaluatinglearningdiagnostic} consisting of 897 questions from the test set (c) DiagnosisArena \cite{zhu2025diagnosisarenabenchmarkingdiagnosticreasoning} consisting of 915 questions.
FV-8B-Diag demonstrates strong gains over the base model, with an improvement of 5.6\% on MedQA and a 3.2\% increase on DiagArena, a challenging benchmark designed to evaluate diagnostic reasoning. \textbf{Further training towards clinical reasoning objectives leads to additional improvements} over FV-8B-Diag, as shown by FV-8B, which achieves consistent gains across all benchmarks. Notably, it attains a 3\% improvement on MedCase, a dataset comprising case-study-based questions that require deeper clinical reasoning.


\begin{table}[ht]
\centering
\small
\setlength{\tabcolsep}{4pt} 
\caption{Accuracy (\%) of LLMs across diagnostic reasoning benchmarks}
\label{tab:medqa_mmlu_scores}

\begin{tabular}{lccc}
\hline
\textbf{Model} & \textbf{MedQA} & \textbf{MedCase} & \textbf{DiagArena} \\
\hline
Qwen3-8B & 79.8 & 32.0 & 33.1 \\
FV-8B-Diag & 85.4 (+5.6\%) & 32.2 (+0.2\%) & 36.3 (+3.2\%) \\
FV-8B & 86.2 (+6.4\%) & 35.0 (+3.0\%) & 37.5 (+4.4\%) \\
\hline
\end{tabular}

\end{table}

\subsection{Clinical Healthcare Reasoning}

\begin{table}[t]
\centering
\caption{Accuracy(\%) of LLMs on healthbench-hard. \textbf{Bold}/\underline{Underline} denote best/second-best per benchmark}
\resizebox{\columnwidth}{!}{%
\begin{tabular}{l|c|c}
\hline
\textbf{LLM} & \textbf{Score} \textbf{(w/o} & \textbf{Score} \\
 & \textbf{sys\_prompt)} & \textbf{(w sys\_prompt)} \\
\hline
\multicolumn{3}{c}{\textbf{Closed-Source Models}} \\
\hline
o4-mini & 17.6 & 23.2 \\
o3 & 31.6 & 35.1 \\
gpt-5.2-thinking$^{*}$ & 38.58 & 43.6 \\
gpt-5-thinking$^{*}$ & \underline{39.95} & \underline{43.79} \\
gpt-5-mini-thinking & \textbf{41.15} & \textbf{44.6} \\
\hline
\multicolumn{3}{c}{\textbf{Open-Source Models}} \\
\hline
Qwen3-4B & 6.3 & 8.56 \\
Qwen3-8B & 10.2 & 13.1 \\
Medgemma-27B & 11.8 & 15.5 \\
II-Medical-8B & 13.2 & 17.95 \\
Qwen3-14B & 11.6 & 15.8 \\
Gpt-oss-20B & 13.5 & 18.77 \\
Qwen3-32B & 13.9 & 18.1 \\
Qwen3-30B-A3B & 18.7 & 23.42 \\
Qwen3-Next-80B-A3B & 26.2 & 30.49 \\
Gpt-oss-120B & 27.5 & 35.52 \\
Baichuan-M2 & \underline{34.7} & 37.92 \\
\hline
\multicolumn{3}{c}{\textbf{Ours}} \\
\hline
\rowcolor{blue!8}
FV-4B & 31.44 & 35.3 \\
\rowcolor{blue!8}
FV-8B & 30.3 & \underline{41.15} \\
\rowcolor{blue!8}
FV-30B-A3B & \textbf{37.75} & \textbf{50.1} \\
\hline
\end{tabular}}
\begin{tablenotes}
\small
\item * The model card for gpt-5-thinking and gpt-5.2-thinking show 46.2 and 42.0 respectively but evaluation using OpenAI's API using the standard script yields 38.58 and 39.95 respectively.
\end{tablenotes}
\label{tab:hb_hard_scores}
\end{table}

Table~\ref{tab:hb_hard_scores} summarizes the performance of a diverse set of open-source and closed-source models on the \textbf{HealthBench-Hard} subset under two evaluation settings: with and without a system prompt. HealthBench-Hard is specifically designed to stress-test \emph{clinical reasoning}, requiring models to handle nuanced medical contexts, avoid hallucinations, and maintain coherent causal reasoning over long conversational trajectories.

Under this stringent evaluation, FV-30B-A3B surpasses all other evaluated models by a clear margin, achieving \textbf{50.1\%} on HealthBench-Hard in the system-prompt setting and establishing a new state of the art under the prescribed protocol. Remarkably, even without a system prompt, FV-30B-A3B matches the performance of \textit{gpt-5-thinking} and exceeds significantly larger open models such as \textit{gpt-oss-120B}, as well as proprietary systems including \textit{gpt-5.2-thinking} and OpenAI’s \textit{o3}. These results underscore that the gains arise from improved reasoning capability rather than prompt engineering or scale alone.

Importantly, the proposed training methodology yields consistent improvements across the entire FV model family. We observe absolute gains of approximately 20-25\% without a system prompt and 25-27\% with a system prompt on HealthBench-Hard. 

Table~\ref{tab:ood_scores} presents the evaluation of the Fathom-Vaidya models on three additional benchmarks: (a) GPQA-Biology \cite{rein2023gpqagraduatelevelgoogleproofqa}, for which we select 71 questions categorized under the ``Biology'' domain by the dataset; and (b) Healthbench-Professional (Hb-Pro), comprising of 525 scenarios similar to how doctors/medical experts commonly use AI for their day-to-day tasks. All FV models show 12-16\% improvement on Hb-pro highlighting significant out-of-distribution generalization. FV-30B-A3B surpasses Grok-4.2 (36.1) and approaches physician-baseline (43.7) on Hb-pro. Further, we observe performance gains of 5\% on GPQA-Biology for FV-4B and FV-8B. This highlights the robustness of the approach and suggests that optimizing for principled clinical reasoning transfers effectively across diverse evaluation settings.
\begin{table}[!ht]
\caption{Accuracy (\%) on GPQA-Bio and Hb-Pro.
Parentheses indicate absolute change over the baseline.}
\label{tab:ood_scores}
\centering
\small
\setlength{\tabcolsep}{4pt}
\begin{tabular}{l|ccc}
\hline
\textbf{Model} & \textbf{GPQA-Bio} & \textbf{Hb-Pro} \\
\hline
Qwen3-4B & 56.14 & 15.9 \\
FV-4B & 62.82 (+6.68) & 28.6 (+12.7) \\
\hdashline \\
Qwen3-8B & 61.53 & 17.4 \\
FV-8B & 64.10 (+2.57) & 29.4 (+12.2) \\
\hdashline \\
Qwen3-30B & 69.20 & 23.2 \\
FV-30B & 70.52 (+1.32) & 39.8 (+16.6) \\
\hline
\end{tabular}
\end{table}

\section{Discussion}

\subsection{Rubrics Generation Capability}
We further evaluate the ability of LLMs to generate evaluation rubrics on the HealthBench-Hard dataset. Each conversation is decomposed into multiple evaluation datapoints, one per axis, such that each datapoint contains only the rubrics corresponding to a single evaluation axis. Under this construction, 1,000 conversations yield 3,455 axis-specific datapoints. For each datapoint, we sample 3 HealthBench-Easy conversations from the same thematic category as the corresponding HealthBench-Hard conversation and provide their rubrics as in-context learning examples. The model is then prompted to generate a list of rubrics for the target datapoint. Rubric quality is assessed using an LLM-as-a-judge (GPT-4.1-mini), which measures semantic similarity between the ground-truth rubrics and those generated by the model. The final score for each datapoint is computed as the weighted average over all matched rubrics.

Figure~\ref{fig:rubrics_gen} reports the relative rubric-generation performance of Qwen and Fathom-Vaidya models. Across all three model scales, Fathom-Vaidya consistently outperforms its corresponding baseline by 3–5\%, indicating that training for HealthBench-Hard completion also confers improved capability in generating semantically accurate, axis-aligned evaluation rubrics.

\begin{figure}[t]
    \centering
    \includegraphics[width=1\columnwidth]{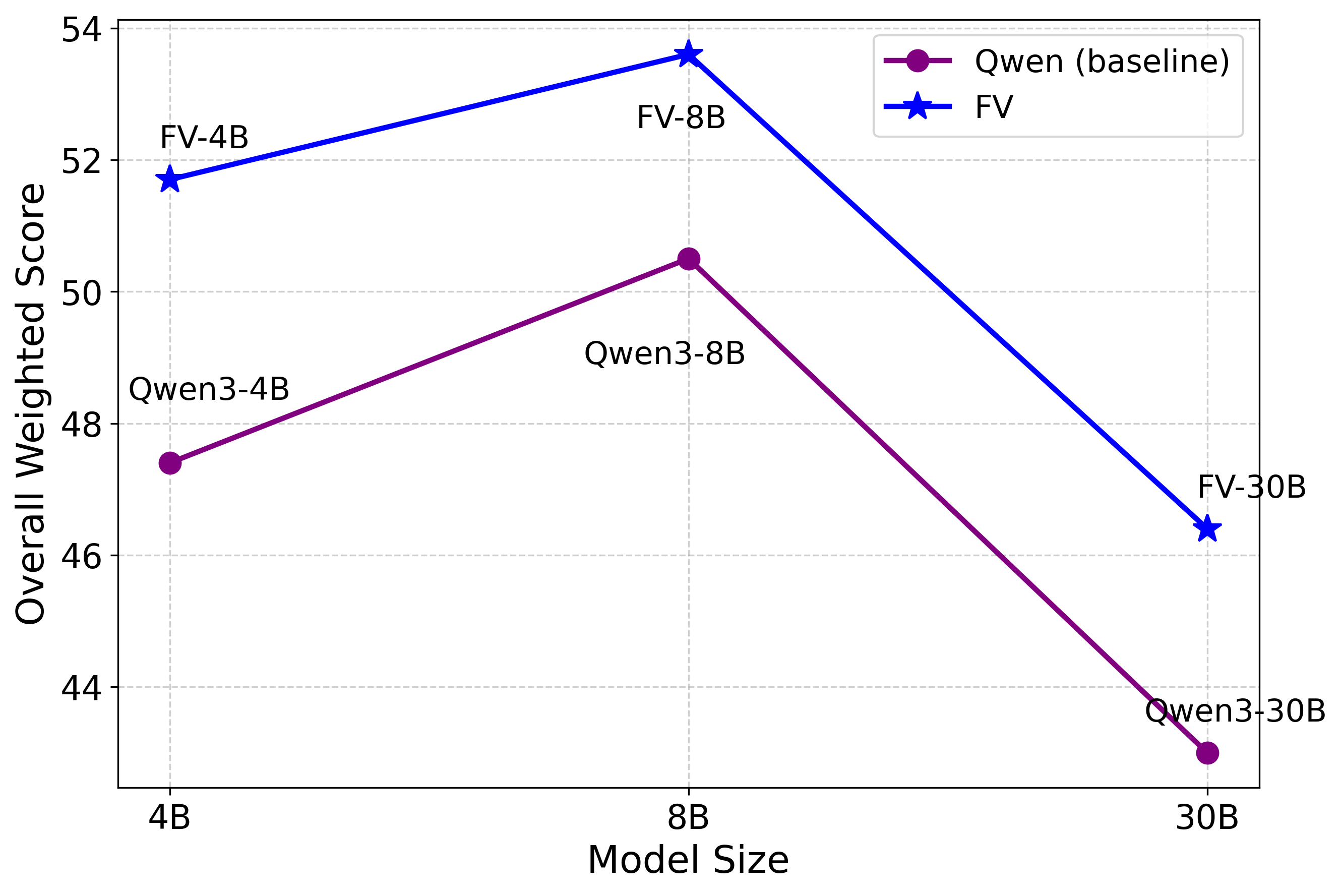}
    \caption{Rubrics generation scores of Qwen \& FV models}
    \label{fig:rubrics_gen}
\end{figure}


\subsection{Ablation on Rubrics Weightage for Diagnostic Reasoning Training}
\label{ablation_medb}

To determine the optimal value of the rubric weightage $\alpha$ described in Section~\ref{rubrics_diag}, we conducted three separate training runs with $\alpha \in \{1.0, 0.7, 0.0\}$. The model obtained from rule-based RL was used as the baseline and fine-tuned for one epoch under each configuration. Based on the observed results in Table~\ref{tab:medb_ablation}, we selected $\alpha = 1.0$ for subsequent training.

\begin{table}[t]
\centering
\setlength{\tabcolsep}{4pt} 

\caption{Ablation study on rubric weightage}
\label{tab:ablation_rubric_epochs}

\begin{tabular}{@{}l l c@{}}
\hline
\textbf{Model} & \textbf{Configuration} & \textbf{MedX score} \\
\hline
FV-8B-Ru1  & $\alpha = 1.0$ & 27.76 \\
FV-8B-Ru0.7 & $\alpha = 0.7$ & 27.31 \\
FV-8B-Ru0   & $\alpha = 0$   & 27.24 \\
\hline
\end{tabular}
\label{tab:medb_ablation}

\end{table}

\subsection{Training Dataset Ablation Study for Clinical Healthcare Reasoning}

We conduct an ablation study to assess the effectiveness of \synhb relative to training solely on the HealthBench-Easy subset. All models follow the training configuration described in Section~\ref{subsec:clinic}, with the exception that each model is trained for five epochs. Figure~\ref{fig:ablation} reports accuracy on HealthBench-Hard under the system prompt setting. Training on \synhb yields improvements exceeding 10\% for the Qwen3-8B and Qwen3-30B-A3B models.

\begin{figure}[h]
    \centering
    \includegraphics[width=1.1\columnwidth]{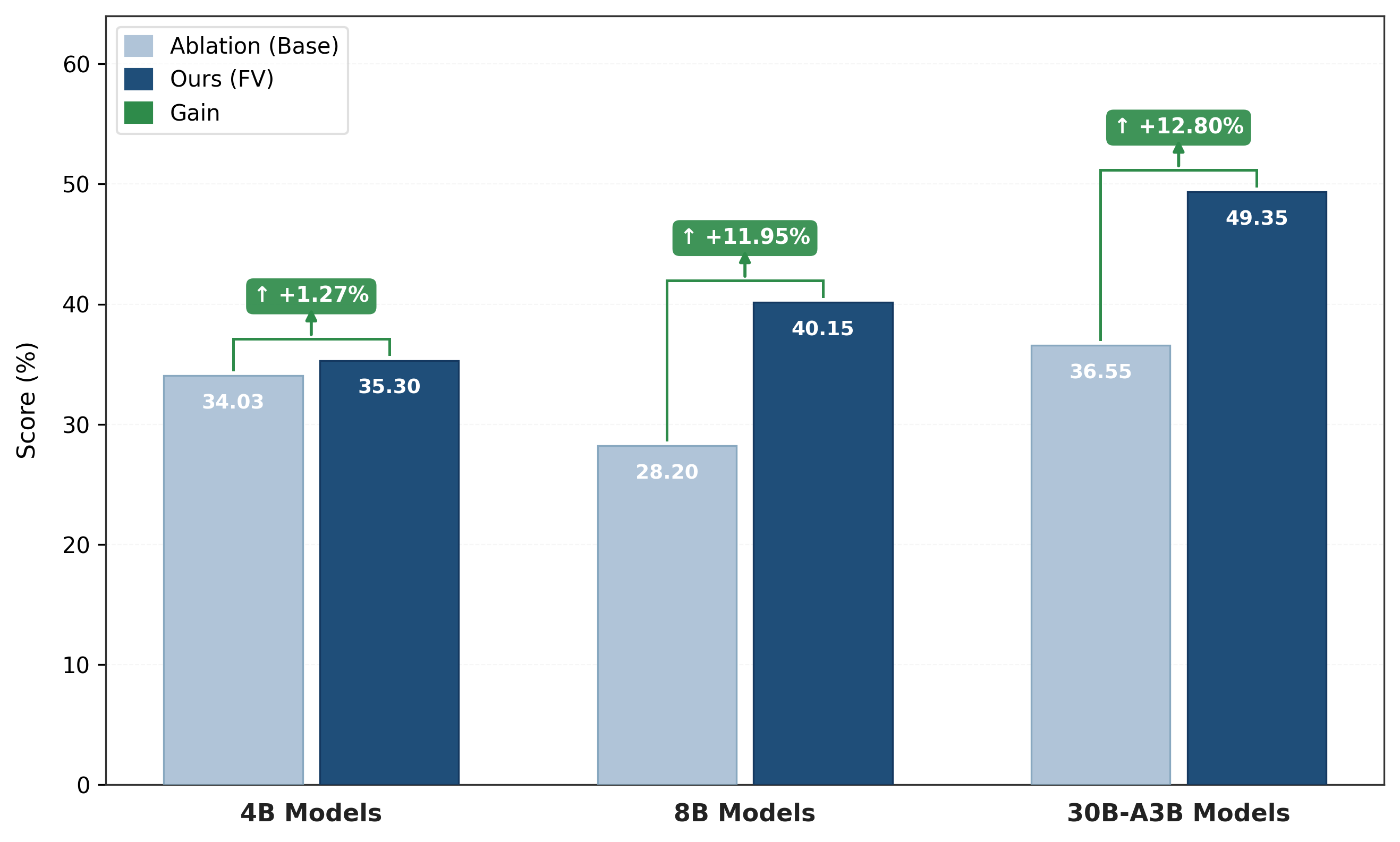}
    \caption{Ablation study comparing base models vs. our FV variants across three scales.Percentages (green) indicate absolute score gains.}
    \label{fig:ablation}
\end{figure}
\section{Conclusion}
Overall, our results demonstrate that carefully structured synthetic data \& rubric-based reinforcement learning can substantially improve both diagnostic \& clinical healthcare reasoning in LLMs. By aligning training signals with the demands of more realistic and challenging medical scenarios, our approach enables meaningful gains and  provides a step towards more reliable and adoptable medical language models. 
\section*{Limitations}
While our approach demonstrates consistent improvements across multiple benchmarks and model scales, several limitations remain. First, experiments are restricted to the Qwen3 family of models, with scales up to 30B parameters, limiting our ability to draw conclusions about the behavior of rubric-based RL at larger scales. Second, all experiments are conducted within a single model family, which shares common architectural choices, tokenizer and pretraining data. As a result, the generality of our approach across diverse LLM architectures remains to be established. Finally, our evaluation is limited to text-based medical reasoning. While benchmarks such as MedXpertQA include multimodal variants, we do not evaluate on settings involving medical imaging, reports, or real world clinical data, which may be critical for deployment in practical healthcare workflows.
\bibliography{custom}

\appendix
\onecolumn
\section{Theme \& Axes Analysis}

\begin{figure*}[!ht]
    \centering
    \begin{subfigure}{0.3\textwidth}
        \centering
        \includegraphics[width=0.9\linewidth]{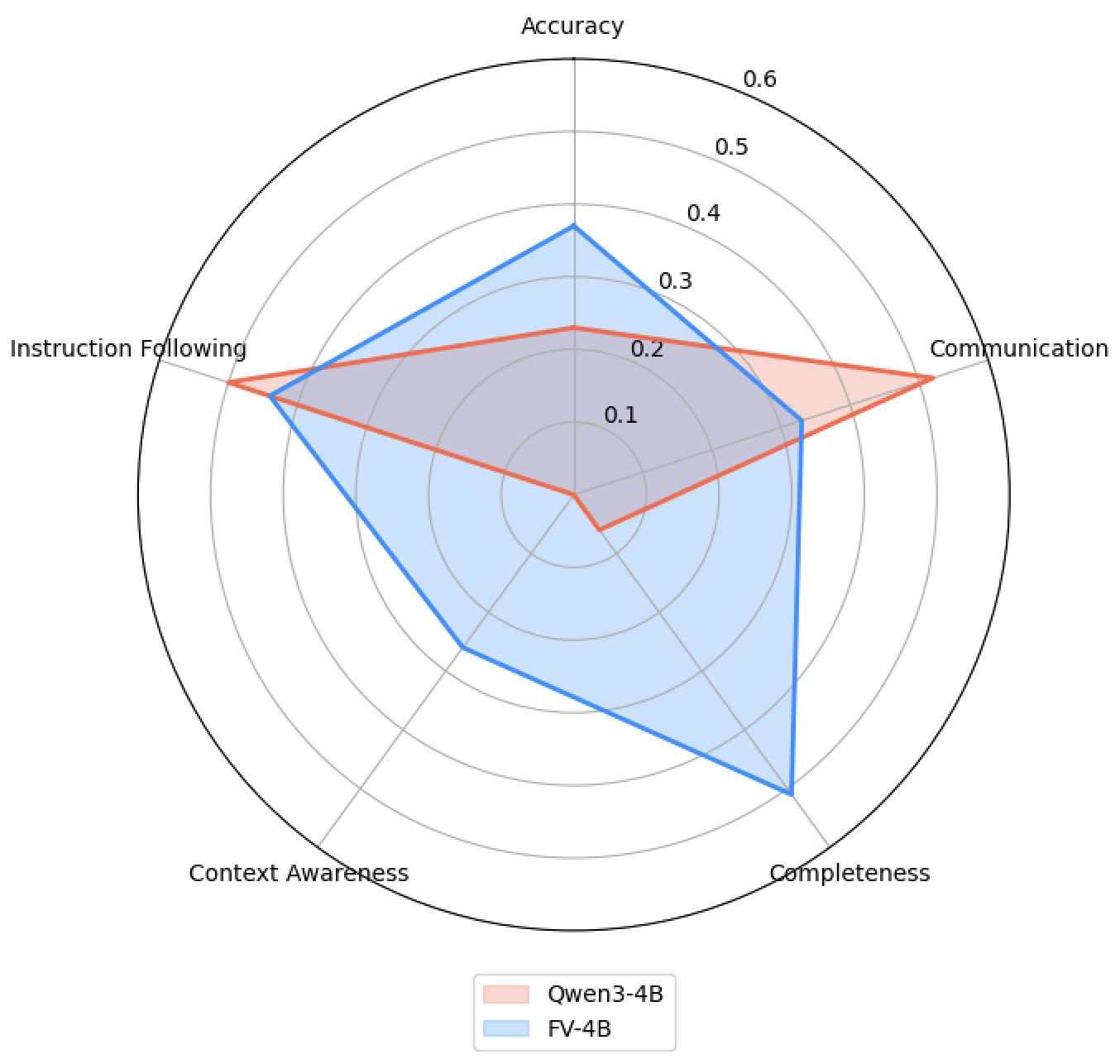}
        \caption{4B}
        \label{fig:plot1a}
    \end{subfigure}
    \hfill
    \begin{subfigure}{0.3\textwidth}
        \centering
        \includegraphics[width=0.9\linewidth]{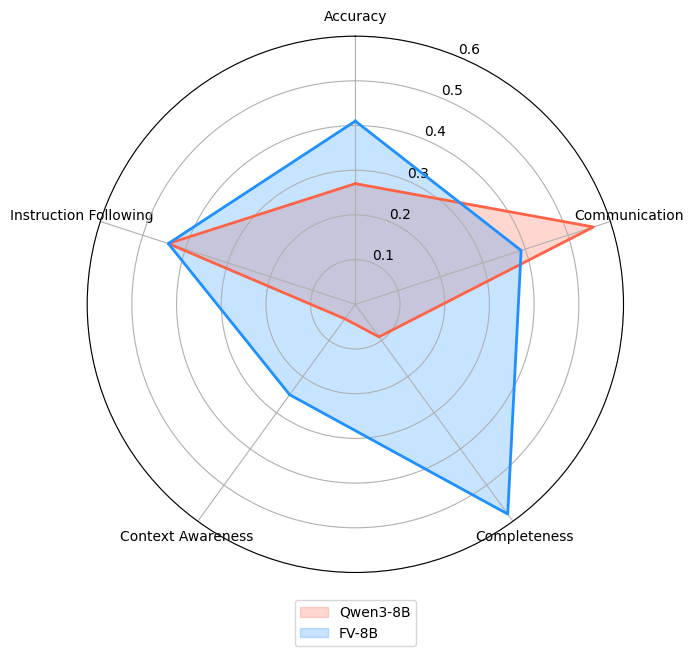}
        \caption{8B}
        \label{fig:plot2a}
    \end{subfigure}
    \hfill
    \begin{subfigure}{0.3\textwidth}
        \centering
        \includegraphics[width=0.9\linewidth]{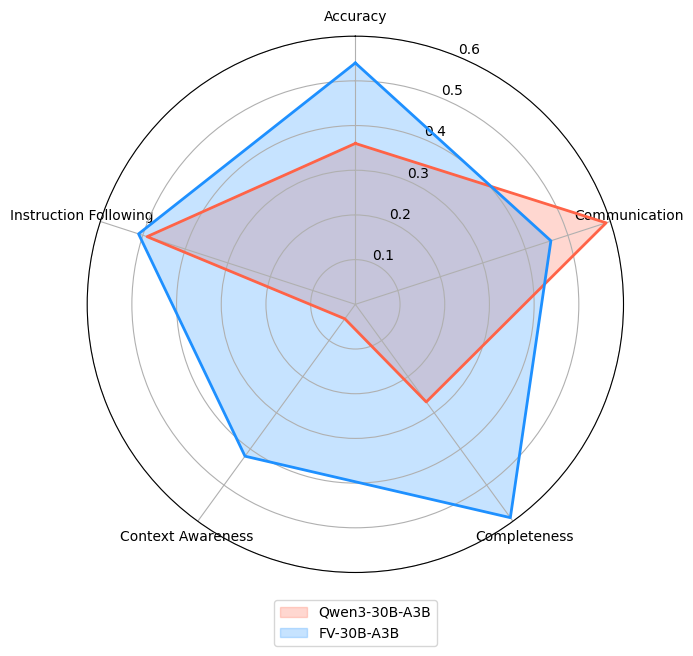}
        \caption{30B-A3B}
        \label{fig:plot3a}
    \end{subfigure}

    \caption{Analysis of Healthbench-hard scores across rubric axes for baseline and Fathom-Vaidya}
    \label{fig:axes}
\end{figure*}

\begin{figure*}[!ht]
    \centering
    \begin{subfigure}{0.3\textwidth}
        \centering
        \includegraphics[width=0.9\linewidth]{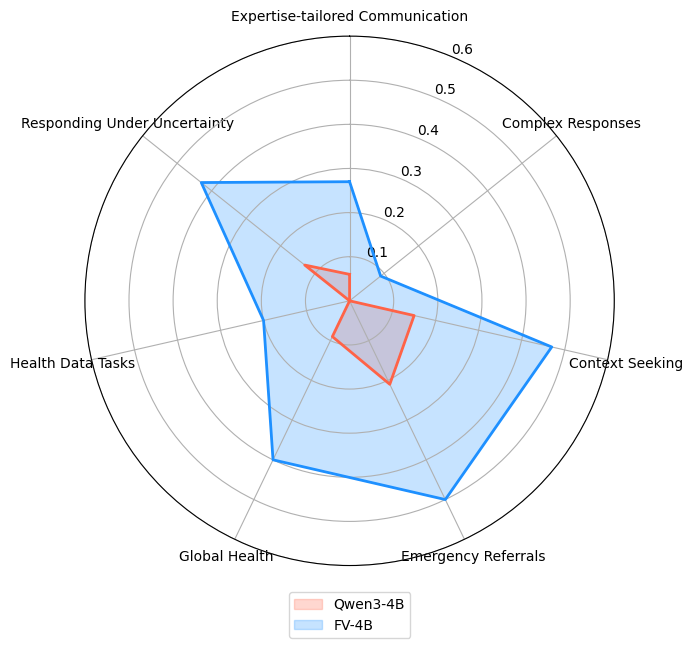}
        \caption{4B}
        \label{fig:plot1}
    \end{subfigure}
    \hfill
    \begin{subfigure}{0.3\textwidth}
        \centering
        \includegraphics[width=0.9\linewidth]{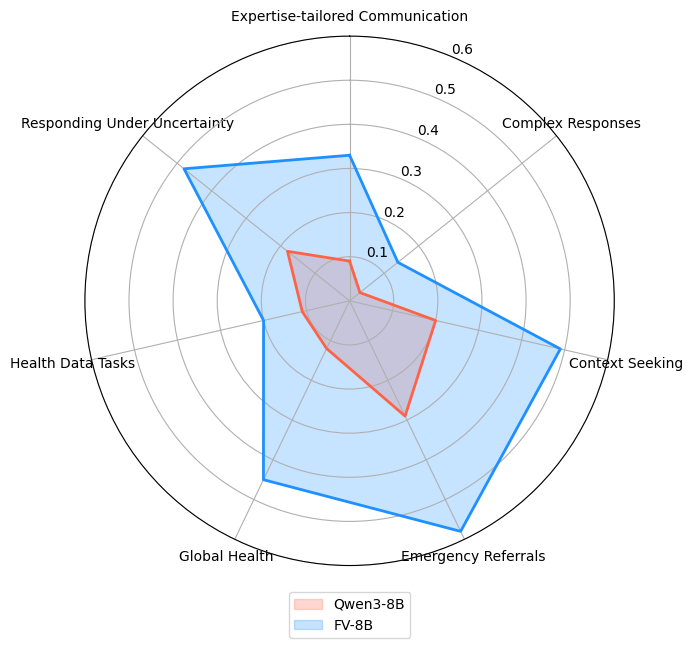}
        \caption{8B}
        \label{fig:plot2}
    \end{subfigure}
    \hfill
    \begin{subfigure}{0.3\textwidth}
        \centering
        \includegraphics[width=0.9\linewidth]{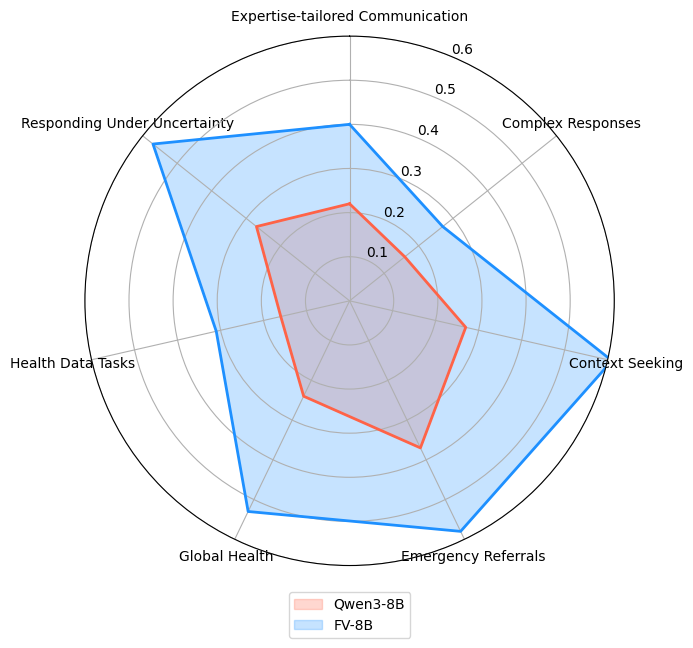}
        \caption{30B-A3B}
        \label{fig:plot3}
    \end{subfigure}

    \caption{Analysis of HealthBench-hard scores across conversation themes for baseline and Fathom-Vaidya}
    \label{fig:themes}
\end{figure*}
Figure~\ref{fig:axes} compares the performance of the baseline and Fathom-Vaidya (FV) models across five evaluation axes. Training results in substantial gains on the completeness and context awareness axes, with absolute improvements of 25–30\%. Model accuracy also improves by approximately 15\%. 

Figure~\ref{fig:themes} shows the performance of baseline and FV models across the seven different themes. While, all the seven themes show an improvement, we observe more than 20\% increase in the case of communication, while Context-seeking, Emergency referrals, Global health and Responding under uncertaininty witness ~30\% enhancement in metrics.

\section{Clinical Evaluation of \synhb and Error Rubrics used in Diagnostic Q/A}
\label{appendix:clinical_eval}
To ensure the quality of our curated synthetic dataset, we collaborated with a panel of 8 expert clinicians, each having atleast 3 years of clinical practice. A random subset of 500 samples from the \synhb dataset was selected for review. Each sample was evaluated by three clinicians, and the final score for each conversation and rubric was recorded only after consensus was reached among them. \textbf{497} conversations out of \textbf{500} (\textbf{99.4\%}) were deemed valid, and \textbf{5441} out of \textbf{5570} individual rubrics (\textbf{97.7\%}) were marked as relevant.

For the diagnostic Q/A evaluation, we randomly selected 50 questions along with their corresponding answers generated by Qwen-3-8B. All identified error rubrics for these questions were subsequently verified through clinical review following consensus among three doctors. A mean agreement of \textbf{92.2\%} was observed between human clinicians and GPT-5 (thinking), based on the adjudication of \textbf{257} error rubrics.

\begin{tcolorbox}[
    enhanced,
    colback=Beige,
    colframe=Sienna,
    title=Conversation Evaluation Instructions (Clinical Healthcare Reasoning),
    arc=4mm,
    boxrule=0.8pt,
    drop shadow=gray!40,
    fonttitle=\bfseries
]
The conversation in \texttt{conversation.txt} should be evaluated to check if it is logically coherent and talks about a single medical topic. For example, a conversation on headache suddenly steering into back-pain without any reference is logically incoherent.

\begin{center}
\begin{tabular}{cl}
\hline
\textbf{Score} & \textbf{Instruction} \\
\hline
0 & If the conversation feels logically incoherent \\
1 & If the conversation is logically coherent \\
\hline
\end{tabular}
\end{center}

Note: Please do not evaluate the conversation based on the AI responses whether it asked the right diagnostic questions or not. This is more of a coherence test.
\end{tcolorbox}

\begin{tcolorbox}[
    enhanced,
    breakable,
    colback=Beige,
    colframe=Sienna,
    title=Rubrics Evaluation Instructions (Clinical Healthcare Reasoning),
    arc=4mm,
    boxrule=0.8pt,
    shadow={0.5mm}{-0.5mm}{0mm}{gray!30},
    fonttitle=\bfseries,
    before upper={\tolerance=10000\relax},
]
Each conversation in the file \texttt{conversation.txt} ends with \texttt{USER}. It is expected that the AI model should frame a response based on that conversation to answer the last query. A sample response is attached for your reference in the file \texttt{sample\_response.txt}.

The rubrics are subjective criteria that the final response from the AI model should satisfy. Now rubrics are of two types:

\textbf{a.} Polarity ``included'' means that criterion should be included in the response. Eg in case of heart attack a rubric stating, ``Should advise to call the ambulance immediately'' with polarity ``included'' means model must advise to call the ambulance.

\textbf{b.} Polarity ``excluded'' means that criterion should be excluded in the response. Eg in case of heart attack a rubric stating, ``Does not advise to call the ambulance immediately'' with polarity ``excluded'' means model failing to advise is wrong and should be penalised. This is a relevant rubric and should be marked 1.\\

Each rubric should be carefully scored according to the following:

\begin{center}
\begin{tabular}{c p{0.75\linewidth}}
\hline
\textbf{Score} & \textbf{Instruction} \\
\hline
0 & If the rubric is not relevant to the final response and does not meaningfully contribute to the evaluation. \\
1 & If the rubric is relevant to the final response that the AI model should produce. \\
\hline
\end{tabular}
\end{center}

\noindent
Note: Please consider the polarity of the rubric carefully before marking.
\end{tcolorbox}

\begin{tcolorbox}[
    enhanced,
    breakable,
    colback=Beige,
    colframe=Sienna,
    title=Error Rubrics Evaluation Instructions (Diagnostic Reasoning),
    arc=4mm,
    boxrule=0.8pt,
    shadow={0.5mm}{-0.5mm}{0mm}{gray!30},
    fonttitle=\bfseries,
    before upper={\tolerance=10000\relax},
]
There are a total of 50 questions where each question has a dedicated folder. Each folder is named as question-{id} like question-1, question-2 and so on uptil question-50. Each folder has 2 items: 
\begin{enumerate}
\item \textit{analysis.txt}: Containing the question, the candidate answer (generated by the LLM), ground truth and the analysis by the Judge. 

\item \textit{rubrics.csv}: The doctors should annotate this file as described below 
\end{enumerate}
Each row in the evaluation sheet contains the following fields:

\begin{itemize}
    \item \texttt{quote} --- The exact quote from the candidate answer that the Judge identified as erroneous.
    \item \texttt{explanation} --- The Judge's explanation of why the quoted text is incorrect.
    \item \texttt{severity} --- The severity of the error: \texttt{1} (minor) or \texttt{2} (major).
    \item \texttt{explanation\_agreement} --- To be filled by the annotator to indicate agreement with the Judge's explanation.
    \item \texttt{severity\_agreement} --- To be filled by the annotator to indicate agreement with the Judge's severity rating.
\end{itemize}

\noindent Annotators should record their scores in the \texttt{explanation\_agreement} and \texttt{severity\_agreement} columns according to Tables~\ref{tab:eval_agreement} and~\ref{tab:severity_agreement} respectively.

\paragraph{Evaluation Agreement}

\begin{center}
\small
\begin{tabular}{cp{0.72\linewidth}}
\toprule
\textbf{Score} & \textbf{Instruction} \\
\midrule
0 & You disagree with the error pointed out by the Judge. \\
1 & You agree with the error pointed out by the Judge. \\
\bottomrule
\end{tabular}
\captionof{table}{Scoring rubric for evaluation agreement}
\label{tab:eval_agreement}
\end{center}

\paragraph{Severity Agreement}

\begin{center}
\small
\begin{tabular}{cp{0.72\linewidth}}
\toprule
\textbf{Score} & \textbf{Instruction} \\
\midrule
0 & You disagree with the error pointed out by the Judge, \emph{or} you disagree with the severity assigned to the error. \\
1 & You agree with both the error and the severity rating assigned by the Judge. \\
\bottomrule
\end{tabular}
\captionof{table}{Scoring rubric for severity agreement}
\label{tab:severity_agreement}
\end{center}

\end{tcolorbox}

\section{Justification for Negative Polarity Rubrics}
\label{appendix:neg_rubrics}
Fix a rubric $j$ with positive value $a>0$. For any conversation, let
$s_j^*$ denote the score obtained by the model excluding rubric $j$, and
$S_j^*$ denote the maximum attainable score excluding rubric $j$, where
$0 \le s_j^* \le S_j^*$ and $S_j^*>0$.
When rubric $j$ has positive polarity, its maximum contribution is $a$
and the model receives $g_j^{(+)} \in \{0,a\}$.
When rubric $j$ has negative polarity, its maximum contribution is $0$
and the model receives $g_j^{(-)} \in \{0,-a\}$.
The normalized conversation score is defined as
\[
\mathrm{Score} = \frac{s_j^* + g_j}{S_j^* + m_j},
\]
where $m_j$ denotes the maximum contribution of rubric $j$.

\begin{proposition}
Replacing a positive rubric $j$ with its negative-polarity contrapositive
never increases the normalized conversation score, and strictly decreases
it unless $s_j^* = S_j^*$.
\end{proposition}

\begin{proof}
We consider two cases.

\medskip
\noindent\textbf{Case (a): The model satisfies the positive criterion.}
In this case,
\[
g_j^{(+)} = a, \quad m_j^{(+)} = a,
\]
and the normalized score is
\[
\mathrm{Score}_{+} = \frac{s_j^* + a}{S_j^* + a}.
\]
After flipping polarity, the model avoids the negative criterion, yielding
\[
g_j^{(-)} = 0, \quad m_j^{(-)} = 0,
\]
and the normalized score becomes
\[
\mathrm{Score}_{-} = \frac{s_j^*}{S_j^*}.
\]
We have
\[
\frac{s_j^*}{S_j^*} \le \frac{s_j^* + a}{S_j^* + a}
\iff
s_j^* a \le a S_j^*
\iff
s_j^* \le S_j^*.
\]
Thus $\mathrm{Score}_{-} \le \mathrm{Score}_{+}$, with strict inequality
whenever $s_j^* < S_j^*$.

\medskip
\noindent\textbf{Case (b): The model does not satisfy the positive criterion.}
In this case,
\[
g_j^{(+)} = 0, \quad m_j^{(+)} = a,
\]
and the normalized score is
\[
\mathrm{Score}_{+} = \frac{s_j^*}{S_j^* + a}.
\]
After flipping polarity, the model violates the negative rubric and incurs
a penalty:
\[
g_j^{(-)} = -a, \quad m_j^{(-)} = 0,
\]
resulting in the normalized score
\[
\mathrm{Score}_{-} = \frac{s_j^* - a}{S_j^*}.
\]
We compare the two scores:
\begin{equation}
\begin{aligned}
\frac{s_j^* - a}{S_j^*} &\le \frac{s_j^*}{S_j^* + a} \\
&\iff (s_j^* - a)(S_j^* + a) \le s_j^* S_j^* \\
&\iff s_j^* \le S_j^* + a.
\end{aligned}
\end{equation}
Since $s_j^* \le S_j^*$ and $a>0$, the inequality always holds, implying
$\mathrm{Score}_{-} < \mathrm{Score}_{+}$.

Combining both cases, flipping a positive rubric to negative polarity
never increases the normalized score and strictly decreases it except in
the degenerate case $s_j^* = S_j^*$.
\end{proof}

\section{Prompts Store}

\subsection{System Prompt for Healthbench}
\label{appendix:hb_prompts}
\noindent
\fbox{%
\begin{minipage}{0.96\linewidth}
\small

\ttfamily
You are a world-renowned medical expert assisting a patient with a
medical issue in a conversational setting.

\vspace{0.6em}

\normalfont\textsc{\textbf{Communication}}\\
\ttfamily
Use clear, structured, and concise language. Tailor the vocabulary,
level of detail, and tone to suit the user’s level of understanding and
the medical context. Highlight important points for emphasis when
appropriate.

\vspace{0.9em}

\normalfont\textsc{\textbf{Instruction Following}}\\
\ttfamily
Follow any explicit or implicit instructions provided by the user,
including specific requests on how to format or frame your response.

\vspace{0.9em}

\normalfont\textsc{\textbf{Accuracy}}\\
\ttfamily
Ensure that your information is medically accurate, based on clinical
consensus or evidence-based guidelines. Acknowledge areas of medical
uncertainty if relevant. Your recommendations should align with the
latest guidelines.

\vspace{0.9em}

\normalfont\textsc{\textbf{Context Awareness}}\\
\ttfamily
Be mindful of the user’s role, background, and the context of the
conversation. Adapt your response accordingly, and ask clarifying
questions if the information provided is insufficient to ensure safe
and helpful guidance. Explain medical terms when necessary or avoid
them when possible.

\vspace{0.9em}

\normalfont\textsc{\textbf{Completeness}}\\
\ttfamily
Fully address the patient’s question or concern. If multiple components
or underlying issues are possible, cover all relevant aspects. Clearly
indicate symptoms that warrant urgent care, further testing, referral
to a specialist, or contacting emergency services, with examples when
appropriate.

\end{minipage}
}

\subsection{Prompts used in the creation of synthetic conversations}
\label{appendix:syn_conv_prompts}
\subsubsection{Analysis Prompt for Critic}
\fbox{%
\begin{minipage}{0.96\linewidth}
\small
\textsc{\textbf{Task Description.}}\\[0.5em]
\ttfamily
You are given a medical case scenario reported by the user. Your task is
to structurally format the case into a multi-turn patient--doctor
conversation by determining the ideal number of conversational turns.

\vspace{1.0em}

\normalfont\textsc{\textbf{Number of Turns.}}\\[0.5em]
\ttfamily
Each set of patient--doctor interaction constitutes one turn. The
following is an example of a two-turn conversation:\\[0.5em]
"user": Hi! I have fever from yesterday\\
"doctor": Okay, any other symptoms?\\
"user": No\\
"doctor": Seems like a viral infection.

\vspace{1.0em}

\normalfont\textsc{\textbf{Instructions.}}\\[0.5em]
\ttfamily
1. Each turn must represent a meaningful and progressive medical
interaction. Do not add a turn if it is not required.\\[0.5em]
2. The total number of turns must be between 1 and 5.\\[0.5em]
3. If the user input is a simple direct query, respond with a single turn
only.\\[0.5em]
4. If more than one turn is required, briefly describe what each turn
should include.

\vspace{1.0em}

\normalfont\textsc{\textbf{Case Input.}}\\[0.5em]
\ttfamily
Analyze the case below:\\[0.5em]
\{case\}

\vspace{1.0em}

\normalfont\textsc{\textbf{Output Format.}}\\[0.5em]
\ttfamily
Strictly output a JSON object in the following format:\\[0.5em]
\{\\
\ \ "num\_turns": <number of turns>,\\
\ \ "reasoning": <your reasoning>\\
\}

\end{minipage}
}

\subsubsection{Conversation Generation Prompt for Actor given case scenario}
\noindent
\fbox{%
\begin{minipage}{0.96\linewidth}
\small

\textsc{\textbf{Task Description}}\\[0.5em]
\ttfamily
You are given a medical case by a user. Your task is to split the case
into a structured conversation consisting of patient-doctor
interactions.

\vspace{1.0em}

\normalfont\textsc{\textbf{Example (Two-Turn Conversation).}}\\[0.5em]
\ttfamily
[\\
\ \ \{"role": "user", "content": "Hi! I have fever from yesterday"\},\\
\ \ \{"role": "doctor", "content": "Okay, any other symptoms?"\},\\
\ \ \{"role": "user", "content": "No, just fever"\},\\
\ \ \{"role": "doctor", "content": "Seems like viral infection"\}\\
]

\vspace{1.0em}

\normalfont\textsc{\textbf{Instructions}}\\[0.5em]
\ttfamily
1. Follow the conversation analysis and split the case into conversations
accordingly.\\[0.5em]
2. User responses should be naturalistic and consistent with the
specified personality.\\[0.5em]
3. Keep the conversation precise and informal; unnecessary details may
be omitted.\\[0.5em]
4. Strictly output a list of dictionaries, where each dictionary contains
a \texttt{"role"} and \texttt{"content"} field.\\[0.5em]
5. Strictly begin the conversation with the user reporting their
complaint.

\vspace{1.0em}

\normalfont\textsc{\textbf{Case Input}}\\[0.5em]
\ttfamily
Case: \{case\}\\[0.5em]
Analysis: \{conv\_analysis\}\\[0.5em]
Patient personality: \{fetch\_patient\_personality()\}

\end{minipage}
}
\subsubsection{Conversation Generation prompt for Actor given ICL examples}
\noindent
\fbox{%
\begin{minipage}{0.96\linewidth}
\small

\textsc{\textbf{Task Description.}}\\[0.5em]
\ttfamily
Your task is to generate a single-turn conversation between a patient
and a doctor on the topic \texttt{\{theme['tag']\}}.
\texttt{\{theme['description']\}}

\vspace{1.0em}

\normalfont\textsc{\textbf{In-Context Examples.}}\\[0.5em]
\ttfamily
\texttt{\{icl\_string\}}

\vspace{1.0em}

\normalfont\textsc{\textbf{Instructions.}}\\[0.5em]
\ttfamily
1. Generate a conversation between the user and the assistant that is
closely related to the specified topic.\\[0.5em]
2. Strictly output a list of dictionaries, where each dictionary
contains the fields \texttt{"role"} and \texttt{"content"}.\\[0.5em]
3. Strictly begin the conversation with the user reporting their
question.

\end{minipage}
}

\subsubsection{Consistency Verification Prompt For Critic}
\noindent
\fbox{%
\begin{minipage}{0.96\linewidth}
\small

\textsc{\textbf{Task Description}}\\[0.5em]
\ttfamily
You are a medical expert classifying conversation transcripts between an
end user or healthcare worker and a chatbot medical assistant. Your goal
is to determine whether the conversation transcript is realistic and
consistent.

\vspace{1.0em}

\normalfont\textsc{\textbf{Evaluation Criteria}}\\[0.5em]
\ttfamily
\textbf{Realistic:} The user queries could plausibly come from a real
layperson or healthcare worker. The tone, formality, style of speaking,
content, and overall interaction are authentic.\\[0.5em]
\textbf{Consistent:} The details within the conversation do not
conflict and are medically and logically possible (e.g., an
80-year-old woman reporting menstrual cramps would be inconsistent).

\vspace{1.0em}

\normalfont\textsc{\textbf{Grading Guidance}}\\[0.5em]
\ttfamily
Be strict when evaluating realism and consistency. If a conversation
appears potentially unrealistic, classify it as unrealistic. If it
appears potentially inconsistent, classify it as inconsistent.

\vspace{1.0em}

\normalfont\textsc{\textbf{Conversation to Classify}}\\[0.5em]
\ttfamily
\{conv\}

\vspace{1.0em}

\normalfont\textsc{\textbf{Output Format}}\\[0.5em]
\ttfamily
Select one of the following options and answer with a single letter
only:\\[0.5em]
(A) Realistic and consistent\\
(B) Unrealistic but consistent\\
(C) Realistic but inconsistent\\
(D) Unrealistic and inconsistent

\end{minipage}
}
\subsubsection{Completeness Verification prompt for Critic}
\noindent
\fbox{%
\begin{minipage}{0.96\linewidth}
\small

\textsc{\textbf{Task Description}}\\[0.5em]
\ttfamily
Grade whether the conversation below contains incomplete user or
assistant messages. An incomplete message is one in which the user or
assistant indicates that some information has been provided, but that
information is not actually included in the message.

\vspace{1.0em}

\normalfont\textsc{\textbf{Definition and Examples}}\\[0.5em]
\ttfamily
For example, a user message such as ``process this note'' is incomplete
unless the note itself is included. Similarly, a message like
``User: The patient’s fragmented data:'' is incomplete unless the
fragmented data is explicitly provided (e.g., age, vitals, symptoms,
medical history).\\[0.5em]
Likewise, an assistant message such as ``Here’s the SOAP note in the
format you requested'' is incomplete unless it also contains the actual
SOAP note. Messages containing placeholders (e.g.,
\texttt{[place draft response here]}) should also be considered
incomplete.

\vspace{1.0em}

\normalfont\textsc{\textbf{Conversation to Classify}}\\[0.5em]
\ttfamily
\{conv\}

\vspace{1.0em}

\normalfont\textsc{\textbf{Output Format}}\\[0.5em]
\ttfamily
Choose one of the following options and answer with a single letter
only:\\[0.5em]
(A) All messages are complete.\\
(B) Any messages are incomplete.

\end{minipage}
}
\newpage
\subsection{Prompts used in the creation of rubrics}
\label{appendix:rubrics_gen_prompts}
\subsubsection{Rubrics Generator Prompt FOR Actor}
\noindent
\fbox{%
\begin{minipage}{0.96\linewidth}
\small

\textsc{\textbf{Task Description}}\\[0.5em]
\ttfamily
Given a conversation between a user and the assistant, your task is to
generate appropriate rubrics to measure the final output produced by the
assistant. This conversation belongs to the topic
\texttt{\{theme['tag']\}}.
\texttt{\{theme['description']\}}
You are required to generate rubrics for the evaluation axis
\texttt{\{axis['tag']\}}.
\texttt{\{axis['description']\}}.

\vspace{1.0em}

\normalfont\textsc{\textbf{In-Context Examples}}\\[0.5em]
\ttfamily
The following example conversations and their corresponding rubrics are
provided as references.\\[0.5em]

\texttt{\#Example Conversation 1}\\
\texttt{\{example\_1\_prompt\}}\\[0.5em]
\texttt{\#Rubrics for Example Conversation 1}\\
\texttt{\{example\_1\_rubrics\}}\\[0.5em]

\texttt{\#Example Conversation 2}\\
\texttt{\{example\_2\_prompt\}}\\[0.5em]
\texttt{\#Rubrics for Example Conversation 2}\\
\texttt{\{example\_2\_rubrics\}}\\[0.5em]

\texttt{\#Example Conversation 3}\\
\texttt{\{example\_3\_prompt\}}\\[0.5em]
\texttt{\#Rubrics for Example Conversation 3}\\
\texttt{\{example\_3\_rubrics\}}

\vspace{1.0em}

\normalfont\textsc{\textbf{Instructions}}\\[0.5em]
\ttfamily
1. Use the provided examples as references when generating rubrics.\\[0.5em]
2. Each rubric must be unique and express a single specific idea; do not
combine multiple ideas into a single rubric.\\[0.5em]
3. Ensure that all rubrics are directly related to the specified
evaluation axis.\\[0.5em]
4. Generate at most \texttt{\{num\_rubrics\}} distinct rubrics and no
more.\\[0.5em]
5. Output only the rubrics as a list, following the format shown in the
examples.\\[0.5em]
6. Do not include any links in your output.

\vspace{1.0em}

\normalfont\textsc{\textbf{Conversation to Evaluate}}\\[0.5em]
\ttfamily
\{convo\}

\vspace{1.0em}

\normalfont\textsc{\textbf{Output Requirement}}\\[0.5em]
\ttfamily
Output the rubrics only. Do not include examples or additional
explanations.

\end{minipage}
}
\subsubsection{Score sanity check prompt for Critic}
\noindent
\fbox{%
\begin{minipage}{0.96\linewidth}
\small

\textsc{\textbf{Task Description}}\\[0.5em]
\ttfamily
You are a medical expert evaluating rubrics generated by an expert for a
conversation. Your goal is to determine whether the assigned rubric
scores are appropriate.

\vspace{1.0em}

\normalfont\textsc{\textbf{Rubric Definitions}}\\[0.5em]
\ttfamily
\textbf{Rubric:} Each rubric consists of a criterion and an associated
score. You must verify that each score is appropriate and lies within
the range from $-10$ to $+10$.\\[0.5em]
\textbf{Positive Rubric:} A rubric that reflects a desired behavior
and must have a positive score (e.g.,
\texttt{\{"criterion": "The response should contain the limit of sweet consumption", "Score": 5\}}).\\[0.5em]
\textbf{Negative Rubric:} A rubric that reflects an undesired behavior
and must have a negative score (e.g.,
\texttt{\{"criterion": "The response does not contain the limit of sweet consumption", "Score": -5\}}).\\[0.5em]
\textbf{Weightage:} Rubrics that are more important should have higher
absolute scores, whether positive or negative.

\vspace{1.0em}

\normalfont\textsc{\textbf{Conversation}}\\[0.5em]
\ttfamily
\{convo\}

\vspace{1.0em}

\normalfont\textsc{\textbf{Rubrics to Evaluate.}}\\[0.5em]
\ttfamily
\{syn\_rubrics\}

\vspace{1.0em}

\normalfont\textsc{\textbf{Output Format}}\\[0.5em]
\ttfamily
Select one of the following options:\\[0.5em]
(A) Appropriate\\
(B) Not appropriate. If any rubric score is inappropriate, briefly state
the reason.

\end{minipage}
}
\subsubsection{Quality check prompt for Critic}
\noindent
\fbox{%
\begin{minipage}{0.96\linewidth}
\small

\textsc{\textbf{Task Description}}\\[0.5em]
\ttfamily
You are a medical expert assessing rubrics generated by an expert. Your
goal is to determine whether the rubrics are complete and unique.

\vspace{1.0em}

\normalfont\textsc{\textbf{Rubric Definition}}\\[0.5em]
\ttfamily
Each rubric consists of a criterion and an associated score. You must
verify that every rubric criterion is complete and that all rubrics are
unique.

\vspace{1.0em}

\normalfont\textsc{\textbf{Judging Criteria}}\\[0.5em]
\ttfamily
-- \textbf{Complete:} Each rubric must form a meaningful and complete
statement. Any incomplete or vague rubric should be flagged as
incomplete (e.g., ``The response should contain'' is incomplete).\\[0.5em]
-- \textbf{Unique:} Each rubric must be distinct. Rubrics that are
repeated or semantically similar should be flagged as not unique.

\vspace{1.0em}

\normalfont\textsc{\textbf{Rubrics to Evaluate}}\\[0.5em]
\ttfamily
\{syn\_rubrics\}

\vspace{1.0em}

\normalfont\textsc{\textbf{Evaluation Guidance}}\\[0.5em]
\ttfamily
Be strict when evaluating both completeness and uniqueness.

\vspace{1.0em}

\normalfont\textsc{\textbf{Output Format}}\\[0.5em]
\ttfamily
Select one of the following options:\\[0.5em]
(A) Complete and unique\\
(B) Incomplete but unique\\
(C) Complete but not unique\\
(D) Incomplete and not unique\\[0.5em]
If you choose (B), (C), or (D), provide a brief justification.

\end{minipage}
}

\end{document}